\documentclass{article}

\pdfoutput=1

\usepackage[utf8]{inputenc} 
\usepackage[T1]{fontenc}    
\usepackage{hyperref}       
\usepackage{url}            
\usepackage{booktabs}       
\usepackage{amsfonts}       
\usepackage{nicefrac}       
\usepackage{microtype}      

\usepackage{natbib}
\usepackage{todonotes}
\usepackage{textcomp}
\usepackage{lipsum}

\usepackage{enumitem}
\usepackage{fancyhdr}

\usepackage{amsmath,amssymb,amsthm,mathrsfs,mathtools}
\usepackage{graphicx}

\usepackage[ruled, vlined, linesnumbered]{algorithm2e}

\usepackage{float}
\usepackage{stmaryrd}

\usepackage{caption}
\usepackage{subcaption}

\usepackage{enumitem}
\setlist[itemize]{leftmargin=*}

\usepackage{tikz}
\usepackage{lipsum}
\usepackage{booktabs} 

\usepackage{multirow}

\usepackage{versions}

\theoremstyle{remark}
\newtheorem{remark}{Remark}
\theoremstyle{definition}
\newtheorem{definition}{Definition}

\newtheorem{proposition}{Proposition}

\DeclareMathOperator*{\argmin}{arg\,min}
\usepackage{xcolor}

\newcommand\calX{\boldsymbol{\mathcal{X}}}
\newcommand\calY{\boldsymbol{\mathcal{Y}}}
\newcommand\calZ{\boldsymbol{\mathcal{Z}}}
\newcommand\calEpsilon{\boldsymbol{\mathcal{E}}}

\newcommand\calD{\boldsymbol{\mathcal{D}}}

\newcommand\calA{\boldsymbol{\mathcal{A}}}

\newcommand\IR{\mathbb{R}}

\newcommand\bz{\boldsymbol{z}}
\newcommand\bx{\boldsymbol{x}}

\newcommand\balpha{\boldsymbol{\alpha}}

\newcommand\bd{\boldsymbol{d}}

\newcommand\bigO{\mathcal{O}}

\newcommand\bZ{\boldsymbol{Z}}

\newcommand\bX{\boldsymbol{X}}

\makeatletter
\newcommand{\eqnum}{\refstepcounter{equation}\textup{\tagform@{\theequation}}}
\newcommand*{\transpose}{%
	{\mathpalette\@transpose{}}%
}
\newcommand*{\@transpose}[2]{%
	\raisebox{\depth}{$\m@th#1\intercal$}%
}

\title{Multivariate Convolutional Sparse Coding \\ with Low Rank Tensor}

\author{
	Pierre Humbert$^1$ \qquad
	Julien Audiffren$^1$ \\
	Laurent Oudre$^2$ \qquad
	Nicolas Vayatis$^1$ \\
	$^1$CMLA, Ecole Normale Supérieure de Cachan, 
	\\ CNRS, Université Paris-Saclay
	94 235 Cachan cedex, France \\
	$^2$L2TI, Université Paris 13, 93430 Villetaneuse, France
}

\begin{document}
	\setlength{\abovedisplayskip}{3pt}
	\setlength{\belowdisplayskip}{3pt}
	\setlength{\belowdisplayskip}{3pt}

	\maketitle

\begin{abstract}
This paper introduces a new multivariate convolutional sparse coding 
based on tensor algebra with a general model enforcing both element-wise sparsity and  low-rankness of the activations tensors.
By using the CP decomposition, this model achieves a significantly more 
efficient encoding of the multivariate signal -- particularly in the 
high order/ dimension setting -- resulting in better performance.
We prove that our model is closely related to the Kruskal tensor 
regression problem, offering interesting theoretical guarantees to our 
setting. Furthermore, we provide an efficient optimization algorithm based on 
alternating optimization to solve this model.
Finally, we evaluate our algorithm with a large range of experiments, 
highlighting its advantages and limitations.
\end{abstract}
	\section{Introduction}

In  recent years, dictionary learning and convolutional sparse coding techniques (CSC) have been successfully applied in a wide range of topics, including image classification \citep{mairal2009supervised,huang2007sparse}, image restoration \citep{aharon2006rm}, and signal processing \citep{ mairal2010online}. 
The main idea behind these representations is to conjointly learn a dictionary containing the patterns observed in the signal, and sparse activations that encode the temporal or spatial locations where these patterns occur.
Previous works have mainly focused on the study of univariate signals or images \citep{garcia2017convolutional}, solving the dictionary learning problem using sparsity constraint on atoms activations.

In many applications ranging from videos to neuroimaging, data are multivariate and therefore better encoded by a tensor structure \citep{zhou2013tensor,cichocki2015tensor}.
This has led in the recent years to an increased interest in adapting statistical learning methods to the tensor framework, in order to efficiently capture multilinear relationships \citep{su2012multivariate}.
In particular, the low rank setting has been the subject of many previous works in the past few years, as an efficient tool to exploit the structural information of the data, particularly in regression problems \citep{zhou2013tensor, rabusseau2016low, li2017near, he2018boosted}.
However, traditional low rank approaches for regression problems are difficult in this setting, as estimating or finding the rank of a tensor with order at least 3 is significantly more complex  than in the traditional matrix case \citep{haastad1990tensor}, and requires new algorithms \citep{li2017near}.
More recently, \citep{he2018boosted} have proposed a new approach --inspired by the SVD decomposition --  that combines both low rank and sparsity constraints in the multivariate regression setting.

In this  article, we propose to extend the classical CSC model to multivariate data by introducing a new tensor CSC model that combines both low-rank and sparse activation constraints.
This model that we call Kruskal Convolutional Sparse Coding (\textbf{K-CSC}) (or Low-Rank Multivariate Convolutional Sparse Coding) offers the advantages of 1) taking into account the underlying structure of the data and 2) using fewer activations to decompose the data, resulting in an improved summary (dictionary) and a better reconstruction of the original multivariate signal.

Our main contributions are as follows. First, we show that under mild assumptions, this new problem can be rewritten as a  Kruskal tensor regression, from which we discuss interesting properties. 
Then, we provide an efficient algorithm, based on an alternating procedure, that solves the \textbf{K-CSC} problem and scales with the number of activation parameters. 
Finally, we evaluate our algorithm experimentally using a large range of simulated and real tensor data, illustrating the advantages and drawbacks of our model.
\section{Low-Rank Multivariate Convolutional Sparse Coding}
\subsection{Preliminaries on tensors}

We first introduce the notation used in this paper and briefly recall some element of tensor algebra (see \citep{kruskal1977three,kolda2006multilinear,kolda2009tensor,sidiropoulos2017tensor} for a more in-depth introduction). 

Across the paper, we use calligraphy font for tensors ($\calX$) bold uppercase letters for matrices ($\bX$) bold lowercase letters for vectors ($\bx$) and lowercase letter for scalars ($x$).
Let $\|\calX\|_F$ and $\|\calX\|_1$ respectively  denotes the Frobenius and the $\ell_1$ norms, and let  $\big< \cdot \big>_F$ be the scalar product associated with  $\| \cdot \|_F$. The symbol $\circ$ refers to the the outer product, $\otimes$ to the Kronecker product, $\odot$ to the Khatri-Rao product, and $\times_m$  denotes the $m$-product. The symbol $\star_{1, \cdots, p}$ refers to the multidimensional discrete convolutional operator between scalar valued functions where the subscript indices are the dimension involved (see Figure \ref{fig:multi_conv}).
When the signal is unidimensional, $\star_1$ reduces to $\star,$ the 1-D discrete convolutional operator.
\begin{proposition}\textit{(CP Decomposition.)}
\label{prop:1}
For any $\calX \in \mathbb{X} \triangleq \mathbb{R}^{n_1} \times \ldots \times \mathbb{R}^{n_p}$, $\exists R>0,$ and, $ \bx^{(i)}_{r} \in  \mathbb{R}^{n_i}$, $1\le i \le p$, $1 \le r \le R$,  such that
	\begin{equation}\label{eq: cp dec}
	\calX =\sum_{r=1}^{R} \bx^{(1)}_{r} \circ \cdots \circ \bx^{(p)}_{r},
	\end{equation}
where $\forall i\geq 2, \lVert \bx^{(i)}_{r} \rVert_F = 1$. The smallest $R$ for which  such decomposition exists is called the \textit{Canonical Polyadic rank} of $\calX$ ($\text{CP-rank}(\calX)$ or rank$(\calX)$ for short), and in this case  \eqref{eq: cp dec} is referred to as the CP decomposition of $\calX$.

\end{proposition}
\begin{definition}\textit{(Kruskal operator.)} 
 With the notation of Proposition \ref{prop:1},the \textit{Kruskal operator} $[\![ \quad\cdot\quad ]\!]$ is defined as
$$
 [\![\bX_1, \cdots, \bX_p]\!] \triangleq \sum_{r=1}^{R} \bx^{(1)}_{r} \circ \cdots \circ \bx^{(p)}_{r}.
$$
where $ \bX_i = \left[\bx^{(i)}_{1}, \ldots, \bx^{(i)}_{R}\right] \in \IR^{n_i\times R},$ $1\le i \le p$.
\end{definition}

\vspace{-1.5em}
\subsection{Model formulation}
Let $\calY \in \mathbb{Y}\triangleq\IR^{n_1 \times \cdots \times n_p}$ be a multidimensional signal and $\calD_1, \cdots, \calD_K$ in $\mathbb{D} \triangleq \IR^{w_1 \times \cdots \times w_p}$ a collection of $K$ multidimensional atoms such that $\forall i, 1 \leq w_i \leq n_i$.  The  \textit{Kruskal Convolutional Sparse Coding model} (\textbf{K-CSC})  is defined as
\begin{equation}
\label{equ:convolve_p_rank}
\calY = \sum_{k=1}^{K} \calD_k \star_{1, \cdots, p} \calZ_k + \calEpsilon,
\end{equation}
where 
A) $\forall 1 \le k \le K$, $\calZ_k \in \mathbb{Z} \triangleq \IR^{m_1 \times \cdots \times m_p}$ (with $m_i = n_i - w_i + 1$) are sparse activation tensors with  CP-rank lower than $R$, with $R$ small, and B)
 $\calEpsilon \in \mathbb{Y}$ is an additive (sub)gaussian noise, whose every component are independent and centered.

\textbf{Advantages of low-rank tensor.} 
The addition of a low rank constraint offers two main advantages.
 First, it reduces the number of unknown activation parameters from $K\big(\prod_{i=1}^{p} m_i\big)$ (unconstrained model) to $K\big(R\sum_{i=1}^{p} m_i\big)$. For instance, in the regression case of typical RGB images of size $n$-by-$n$-by-$3$, the number of parameters decrease to  $R\cdot (2n + 3)$ instead of $3 n^2$, resulting in a better scalability of the problem. Second, it  exploits the structural information of $\calY$, and has already been proved to be effective in various contexts (e.g. \citep{guo2012tensor, liu2013tensor}). For example, previous works have shown that the vectorization of an image removes the inherent spatial structure of it while low rank tensor regression produces more interpretable results \citep{zhou2013tensor}.

\textbf{Relation with separable convolution.} 
The low-rank constraint imposes that each activation  $\calZ_k$ has a $\text{CP-rank}(\calZ_k) \leq R$ and writes as the sum of at most $R$ separable filters (product of multiple one dimensional filters). Problem  \eqref{equ:convolve_p_rank} is therefore a \textit{separable convolution problem}, which allows to use a FFT resolution and to significantly  speed up the calculus of the convolution.  As an example, the complexity of filtering  an $n_1$-by-$n_2$ image with a $w_1$-by-$w_2$ non-separable filter  is  $\bigO(n_1 n_2 w_1 w_2)$  -- instead of  $\bigO(n_1 n_2 (w_1 +w_2) )$ for a separable filter.

\begin{remark}
	For simplicity, we assume in this paper that, $\forall k$, $\text{CP-Rank}(\calZ_k) = R$. However, it is straightforward to extend $\eqref{equ:convolve_p_rank}$ to a model where each tensor $\calZ_k$ have a different CP-rank (i.e. $\text{CP-Rank}(\calZ_k) = R_k$).
\end{remark}

\begin{figure}
	\centering
	\includegraphics[width=0.5\linewidth]{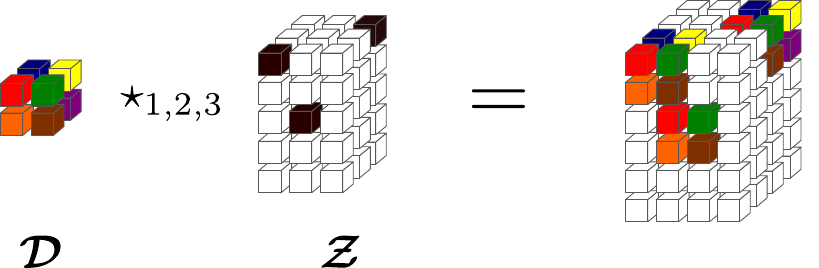}
	\caption{Illustration of the multidimensional convolution with $3$-th order tensors, where each cube represents a dimension and each axis an order. Notice that the result has one additional dimension in each order.}
	\label{fig:multi_conv}
\end{figure}

	\subsection{Kruskal CSC as Tensor Regression}

CSC models are often rewritten as a regression problem by using a circulant dictionary matrix \citep{garcia2018convolutional}. 
In this section, we adopt the same strategy. We show that the \textbf{K-CSC} model \eqref{equ:convolve_p_rank} can be reformulated as an instance of the rank-$R$ Generalized Linear tensor regression Models introduced in \citep{zhou2013tensor} -- also known as Kruskal tensor regression model -- \citep{zhou2013tensor} and then use this new formulation to discuss the properties of the model.
First we introduce the notion of  circulant tensor, which is a generalization of the circulant matrix.
\begin{definition}\textit{(Circulant tensor.)}
	Let  $w_1 < n_1,\ldots, w_p < n_p \in \mathbb{N_*}$ and $\calD \in  \IR^{w_1 \times \cdots \times w_p}$. We define the \textit{(quasi)}-circular tensor generated by $\calD$,
	$\text{$\mathcal{C}$irc}(\calD) \in \IR^{(n_1 \times m_1) \cdots \times (n_p \times m_p)}$,  as follows
	\begin{equation*}
	\text{$\mathcal{C}$irc}(\calD)(\ell_1, k_1, \cdots, \ell_p,k_p) = \left\{
	\begin{array}{ll}
	&0 \quad
	\mbox{if $\exists 1 \le i\le p$  s.t. $\ell_i < k_i$ or $\ell_i\ge k_i+w_i$  } \\
	& \calD(\ell_1-k_1, \cdots, \ell_p-k_p) \quad \mbox{otherwise.}
	\end{array}
	\right.
	\end{equation*}
\end{definition}
In other word, $\text{$\mathcal{C}$irc}(\calD)$ contains the atom $\calD$ which is translated in every directions. The following proposition shows the equivalence between \textbf{K-CSC} model and Kruskal regression. The consequence of this relation are studied in the next Section.

\begin{proposition}\textit{(K-CSC as Kruskal regression.)} \label{prop: multiconv regress}
	Let $\calD_1, \cdots, \calD_K$ in $\mathbb{D}$ be  a collection of $K$ multidimensional atoms. Then, \eqref{equ:convolve_p_rank} is equivalent to 
	\begin{equation}
	\label{equ:MultiD_reg}
	\calY  = \calA(\calZ) + \calEpsilon,
	\end{equation}
	where $\calZ$ is in $\IR^{K} \times \mathbb{Z}$ such that $\calZ_k \triangleq \calZ(k, \cdot, \cdots, \cdot) = \sum_{r=1}^{R} z_{k, r}^{(1)} \circ \cdots \circ z_{k, r}^{(p)}$ and $\calA$ is the linear map defined by 	
	$$\begin{array}{lclc}
	\calA : &\IR^{K} \times \mathbb{Z} &\xrightarrow{\hspace*{.7cm}}& \mathbb{Y}\\
	&\calZ &\xmapsto{\hspace*{.7cm}}& \Big(\sum_{k=1}^{K} \big\langle\tilde{\calD}_{k; i_1, \cdots, i_p}, \calZ_k \big\rangle_F\Big)^{n_1,\ldots,n_p}_{i_1=1,\ldots,i_p=1},
	\end{array}$$
where  $\forall k, \tilde{\calD}_{k; i_1, \cdots, i_p}=\text{$\mathcal{C}$irc}(\calD_k)(i_1, \cdot, i_2, \cdot, \cdots, i_p, \cdot)$.
\end{proposition}
\begin{proof}

The proposition is a consequence of the following equality

\begin{align*}
		\Big(\calD_k \star_{1, \cdots, p} \bz^{(1)}_r \circ \cdots \circ \bz^{(p)}_r\Big)_{i_1, \cdots, i_p} &= \Big\langle \text{$\mathcal{C}$irc}(\calD_k){(i_1, :, i_2, :, \cdots, i_p, :)}, z_{r}^{(1)} \circ \cdots \circ z_{r}^{(p)} \Big\rangle_F \\
		&= \Big(\tilde{\calD}_{k;i_1, \cdots, i_p} \bigtimes^p_{m=1} \bz^{(m)}_r \Big),
\end{align*}
 where 
$ 
\tilde{\calD}_{k;i_1, \cdots, i_p} \bigtimes^p_{m=1} \bz^{(m)}_r  \triangleq \tilde{\calD}_{k;i_1, \cdots, i_p} \times_1 \bz^{(1)}_r \times _2  \bz^{(2)}_r \cdots \times_p \bz^{(p)}_r.
$
\end{proof}

	\section{Model estimation}\label{sec: model estimation}
In order to solve \eqref{equ:convolve_p_rank}, we minimize the associated negative log-likelihood with an additional sparsity constraint -- a common tool in CSC -- on the activation tensors.
This leads to the following minimization problem:

\begin{equation}
\label{equ:CSKTRridge}
\argmin_{ \big(\calD_1,\cdots,\calD_K, \bZ_{1,1}, \cdots,\bZ_{K,p}\big) \in \mathcal{S}}
\Bigg\lVert \calY - \sum_{k=1}^{K} \calD_k \star_{1, \cdots, p} [\![\bZ_{k, 1}, \cdots, \bZ_{k, p}]\!] \Bigg\rVert_F^2 + \sum_{k,\ell} \alpha_{\ell} \lVert \bZ_{k, \ell}\rVert_1,
\end{equation}
where $\forall \ell$ $\alpha_\ell >0$ and $\big(\calD_1,\cdots,\calD_K, \bZ_{1,1}, \cdots,\bZ_{K,p}\big) \in \mathcal{S}$ i.f.f. 
\begin{equation*}
\left\lbrace
\begin{aligned}
& \forall k, \calD_k \in \mathbb{D}, \| \calD_k\|_F \le 1\\
& \forall k,\ell, \bZ_{k,\ell} \in \mathbb{R}^{n_\ell \times R}  \text{ and } \| \left(\bZ_{k,\ell} \right)_{i,\cdot} \|_F = 1
\end{aligned}
\right.
\end{equation*}
Notice that the low rank constraint is embedded in $\mathcal{S}$ by the use of the \textit{Kruskal operator} $ [\![ \quad \cdot \quad ]\!]$ and the fact that   $\forall k,\ell, \quad \bZ_{k,\ell} \in \mathbb{R}^{n_\ell \times R}$. 
This use of the Kruskal operator can be seen as a generalization of the  \textit{Burer-Monteiro heuristic} for matrix  \citep{burer2003nonlinear}.

\paragraph{Multiple CP Decompositions.}
It should be noted that the CP decomposition is known to be unique when it satisfies the Kruskal condition \citep{kruskal1989rank}, but only up to permutation of the normalized factor matrices. Therefore, the  $\bZ_{k,\ell}$ that solve \eqref{equ:CSKTRridge} may not be unique, but they are isolated equivalent minimizers, and thus this problem does not negatively impact the optimization.

\paragraph{Regularizations.} 
In \eqref{equ:CSKTRridge}, the sparsity constraint enforces the sparsity of each element of the CP-decomposition for every activation tensors independently. Hence, the sparsity of each mode can be controlled. In addition, it is possible to add a specific ridge penalization (i.e. $\sum_{\ell=1}^{p} \beta_\ell \sum_{k=1}^{K} \lVert \bZ_{k, \ell}\rVert^2_F$, $(\beta_{1}, \cdots, \beta_{p}) \succeq 0$) to \eqref{equ:CSKTRridge} to limit numerical instabilities, decrease risks of overfitting and ensure the unity of the solution \citep{zhou2013tensor}.

\paragraph{Sparse Low Rank Regression.}
A consequence of Proposition \ref{prop: multiconv regress} is that \eqref{equ:CSKTRridge} can be seen as an instance of a low-rank sparse regression problem. It has been shown that while both properties are desirable, a balance between the two constraints has to be found, as the two regularizations may have adversarial influence on the minimizer   \citep{richard2012estimation}.  This is achieved in our setting by using a Ivanov regularization for the rank (CP-rank $\le R$) and a Tykhonov regularization for the sparsity ($\sum_{k,\ell} \alpha_{\ell} \lVert \bZ_{k, \ell}\rVert_1$): the solution should be as sparse as possible while having a CP-rank lower than $R$.

\subsection*{Solving the optimization problem with AK-CSC}\label{subsec: strategy}
The non-convex problem \eqref{equ:CSKTRridge}(due to the rank constraint) is convex with respect to each $Z$ block $([\bZ_{k, i}, \cdots, \bZ_{k, i}])_i,$ and $([\calD_1, \cdots, \calD_K ])$. Hence, we use a block coordinate strategy to minimize it. Our algorithm, called Alternated \textbf{K-CSC} (\textbf{AK-CSC}), splits the main non-convex problem into several convex subproblems; 1) by freezing $\calD$ and all except one $Z$ block at a time ($\calZ$-step) 2) by freezing only the activation tensor ($\calD$-step).
Algorithm \ref{alg:seq_methodo} presents this process.

\textbf{Activations update, $\calZ$-step.}
In order to solve \eqref{equ:CSKTRridge} with respect to $\calZ$, we proceed as follows: we assume that the dictionary $\calD$ is fixed and we iteratively  solve the problem where all mode except the $\ell$-th one of each activation tensor are constant, for $\ell$ varying between $1$ and $p$.
In other words, for each value of $\ell$, we solve the problem
\begin{equation}
	\label{eq:zstep_1}
	\argmin_{\bZ_{1, \ell}, \cdots, \bZ_{K, \ell}}\Bigg\lVert \calY - \sum_{k=1}^{K} \calD_k \star_{1, \cdots, p} [\![\bZ_{k, 1}, \cdots, \bZ_{k, p}]\!] \Bigg\rVert_F^2 + \alpha_\ell\sum_{k=1}^{K} \lVert \bZ_{k, \ell}\rVert_1 + \beta_\ell \sum_{k=1}^{K} \lVert \bZ_{k, \ell}\rVert^2_F.
\end{equation}
where the $\lVert \cdot \rVert^2_F$ is added to improve the minimization process, as previously discussed.
Without any loss of generality, we set $\ell=1$ in the rest of this section, as the other values of $\ell$ can be treated similarly.
\begin{proposition} The first term of minimization problem \eqref{eq:zstep_1} can be rewritten as 
	\begin{equation*}
		\label{eq:zstep_1_bis}
 \Big\lVert \calY - \sum_{k=1}^{K} \calD_k \star_{1, \cdots, p} [\![\bZ_{k,1}, \cdots, \bZ_{k, p}]\!] \Big\rVert_F^2 = \sum^C_{c} \Big\lVert \tilde{\calY}_{:, c} - \sum^S_{s} \tilde{\calD}_{s, :, c} \star \bz_{s}^{(\ell)} \Big\rVert^2_2,
	\end{equation*}
	with $C=\prod_{i=2}^{p} n_i$, $S=K  R$ and $\bz_{s}^{(\ell)} = \big[z_{1, 1}^{(\ell)}, \cdots, z_{1, R}^{(\ell)}, z_{2,1}^{(\ell)}, \cdots, z_{K, R}^{(\ell)} \big]$.
\end{proposition}
\begin{proof}
	In the following, for all $k$, we denote by $\bar{\calZ_k} = \sum_{r=1}^{R}  \bar{z}_{k, r}^{(1)} \circ \cdots \circ \bar{z}_{k, r}^{(p)} \in \mathbb{Y}$ the tensor where we add $0$ on each dimension to reach the one of $\calZ$.
	\begin{align*}
	&\Big\lVert \calY - \sum_{k=1}^{K} \calD_k \star_{1, \cdots, p} \calZ_k \Big\rVert_F^2
	= \sum_{i_1=1, \cdots, i_p=1}^{n_1, \cdots, n_p} \Big(\calY_{i_1, \cdots, i_p} - \sum_{k=1}^{K}\sum_{r=1}^{R}\sum_{j_1=1}^{w_1}  \bar{z}_{k, r}^{(1)}(i_1 - j_1) \\ &\hspace{15em} \sum_{j_2=1, \cdots, j_p=1}^{w_2, \cdots, w_p}\calD_{k, j_1, \cdots, j_p}  \bar{z}_{k, r}^{(2)}(i_2 - j_2) \cdots  \bar{z}_{k, r}^{(p)}(i_p - j_p)\Big)^2 \\
	\end{align*}
	\begin{align*}
	=& \sum_{i_1=1, \cdots, i_p=1}^{n_1, \cdots, n_p} \Big(\calY_{i_1, \cdots, i_p} - \sum_{k=1}^{K}\sum_{r=1}^{R}\sum_{j_1=1}^{w_1}  \bar{z}_{k, r}^{(1)}(i_1 - j_1) \Big(\calD_{k; j_1, :, \cdots, :} \star_{2, \cdots, p} z_{k, r}^{(2)} \circ \cdots  \circ z_{k, r}^{(p)}\Big)_{i_2, \cdots, i_p} \Big)^2 \\
	=& \sum_{i_2=1, \cdots, i_p=1}^{n_2, \cdots, n_p} \Big\lVert \calY_{:, i_2, \cdots, i_p} - \sum_{k=1}^{K}\sum_{r=1}^{R} \tilde{\calD}_{k, r, :, i_2 \cdots, i_p} \star  z_{k, r}^{(1)} \Big\rVert^2_2 = \sum^C_{c=1} \Big\lVert \calY_{:, c} - \sum^S_{s=1} \tilde{\calD}_{s, :, c} \star z_{s}^{(1)} \Big\rVert^2_2, \\
	\end{align*}
	where $\tilde{\calD}_{k; r, j_1, i_2 \cdots, i_p} = \Big(\calD_{k; j_1, :, \cdots, :} \star_{2, \cdots, p} z_{k, r}^{(2)} \circ \cdots  \circ z_{k, r}^{(p)}\Big)_{i_2, \cdots, i_p}.$
	\end{proof}
From the previous proposition, it is clear that in \eqref{eq:zstep_1}, each subproblem is \textit{a CSC with multichannel dictionary filters and single-channel activation maps} \citep{wohlberg2016convolutional}, i.e. we need to solve 
\begin{equation*}
	\argmin_{\bz_{s}^{(\ell)}} \sum^C_{c} \Big\lVert \tilde{\calY}_{:, c} - \sum^S_{s} \tilde{\calD}_{s, :, c} \star \bz_{s}^{(\ell)} \Big\rVert^2_2 + \alpha_\ell \sum^S_{s} \lVert \bz_{s}^{(\ell)} \rVert_1 + \beta_\ell \sum^S_{s} \lVert \bz_{s}^{(\ell)} \rVert^2_2.
\end{equation*}
Therefore, this Z block step can be solved using standard multi-channel CSC algorithms (see \citep{garcia2018convolutional} for a complete review).

\textbf{Dictionary update, $\calD$-step.}
Given the $K$ activation tensors $(\calZ_k)_k$, the dictionary update aims at improving how the model reconstructs $\calY$ by solving

\begin{equation}
\label{equ:CSKTRridge2}
\argmin_{{\forall k, \calD_k \in \mathbb{D}, \quad  \lVert \calD_k \rVert_F \leq 1}} \Big\lVert\calY - \sum^K_{k=1}\calD_k \star_{1,\cdots,p} [\![\bZ_{k, 1}, \cdots, \bZ_{k, p}]\!] \Big\rVert^2_F.
\end{equation}
This step presents no significant difference with existing methods. This problem is smooth and convex and can be solved using classical algorithms \citep{mairal2010online, yellin2017blood, chalasani2013fast}.

\begin{algorithm}[t]
	\KwIn{Signal $\calY$, $(\alpha_1, \cdots, \alpha_p)$}
	\KwOut{$(\calZ_k)_k$, $(\calD_k)_k$}
	\For{$t = 1, 2, \cdots $}{
		\textbf{-- (Z-step) --} \\
		\For{$m = 1, 2, \cdots, p $}{
			$\tilde{\calY} \longleftarrow \text{unfold}\big(\calY, m\big)$ \;
			\textit{/* Construction of the specific dictionary */} \;
			$s \longleftarrow 1$ \;
			\For{$k = 1, 2, \cdots, K$}{
				$\tilde{\calZ} \longleftarrow [\![\bZ^{(t+1)}_{k,1}, \cdots, \bZ^{(t+1)}_{m-1}, \bZ^{(t)}_{m+1}, \cdots, \bZ^{(t)}_{k, p}]\!]$ \;
				\For{$r = 1, 2, \cdots, R$}{
					\For{$a = 1, 2, \cdots, w_m$}{
							$\tilde{\calD}_{s, a, :} \longleftarrow \text{vectorized}\big(\calD^{(t)}_{k, :, \cdots, a, \cdots, :} \star_{1, \cdots, (m-1), (m+1), \cdots, p} \tilde{\calZ}\big)$  \;
						}
					$s \longleftarrow s + 1$  \;
				}	
			}
			\textit{/* Update of the $m$-th $Z$ block */} \;
			$\bZ_m^{(t+1)} \longleftarrow \text{reshape}\Big(\displaystyle\argmin_{\bz^{(m)}_{s}} \sum^C_{c} \Big\lVert \tilde{\calY}_{:, c} - \sum^S_{s} \tilde{\calD}_{s, :, c} \star \bz^{(m)}_{s} \Big\rVert^2_2 + \alpha_m \lVert \bz^{(m)}_{s} \rVert_1\Big)$  \;
		}
		\vspace{1em}
		\textbf{-- (D-step) --} \\
		$(\calD^{(t + 1)})_k \longleftarrow \displaystyle\argmin_{{\forall k, \calD_k \in \mathbb{D}, \lVert \calD_k \rVert_F \leq 1}} \Bigg\lVert\calY - \sum^K_{k=1}\calD_k \star_{1,\cdots,p} [\![\bZ^{(t+1)}_{1, 1}, \cdots, \bZ^{(t+1)}_{K, p}]\!] \Bigg\rVert^2_F$ \;
	}
	\caption{Alternated Kruskal Convolutional Sparse Coding (AK-CSC)}
	\label{alg:seq_methodo}
\end{algorithm}
\DecMargin{2em}
	\section{Related work}
\textbf{Related models.} With specific choices on the parameters or on the dimension values, the K-CSC model reduces to well-known CSC problems. In the following, we enumerate some of them which also use the (multidimensional) convolutional product.
\begin{itemize}
	\item \textbf{Univariate CSC:} For vector-valued atoms and signals ($p=1$), our model  reduces to
	 the $1$-D Convolutional Dictionary Learning model (CDL). The sparse coding step -- i.e. the minimization on $(\bz_k)_k$ -- is commonly referred as Convolutional Basis Pursuit DeNoising where the two leading approaches are based on the the Fast Iterative
	Shrinkage-Thresholding Algorithm (FISTA) \citep{beck2009fast} and on ADMM \citep{boyd2011distributed}. The dictionary update step -- i.e. the minimization on $(\bd_k)_k$ -- is referred as Method of Optimal Directions (with a constraint on the filter normalization) where, again, the most efficient solutions are based on FISTA and ADMM \citep{garcia2018convolutional}.
	
	\item \textbf{Multivariate CSC:} If $p>1$ and $R= +\infty$ (i.e. no low-rank constraint), our model reduces to the multivariate CDL model also referred as multi-channel CDL. To the best of our knowledge, the only reference to multivariate CSC is \citep{wohlberg2016convolutional}, where the author proposes two models for $3$-channel images. More recently, \citep{garcia2018convolutional} propose scalable algorithms for hyper-spectral images (tensor of order $3$ with a high dimension on the third mode).
	
	\item \textbf{Multivariate CSC with rank-1 constraint:} If $p=2$, $R=1$ and $w_2=1$ (i.e. vector-valued atoms), our model reduces to the one recently presented in \citep{la2018multivariate}. In this work, the authors strongly rely on the rank one constraint to solve the CSC problem and only consider spatio-temporal data.
\end{itemize}

\textbf{Differences with recent tensor based approaches.} 
Some previous works have proposed different approaches to CSC in a tensor setting, albeit without a low rank setting -- a key component of our approach.
In \citep{jiang2018efficient}, the authors introduce a new formulation  based on a t-linear combination (related to the t-product). In
\citep{bibi2017high} they propose a generic CSC model based on tensor factorization strategy called t-SVD which also use the t-product. Notice that, while this new formulation reduces back to 2-D CSC when the tensor order is set to particular sizes, the notion of low-rank on the activations is not considered.
Other tensor-based CSC models enforce the low-rank constraint on the dictionary instead of the activations using a Tucker or a CP decomposition \citep{zhang2017tensor, tan2015tensor} or tensor factorization techniques \citep{huang2015convolutional}.
	\begin{figure}[t]
	\centering
	\includegraphics[width=1.\linewidth]{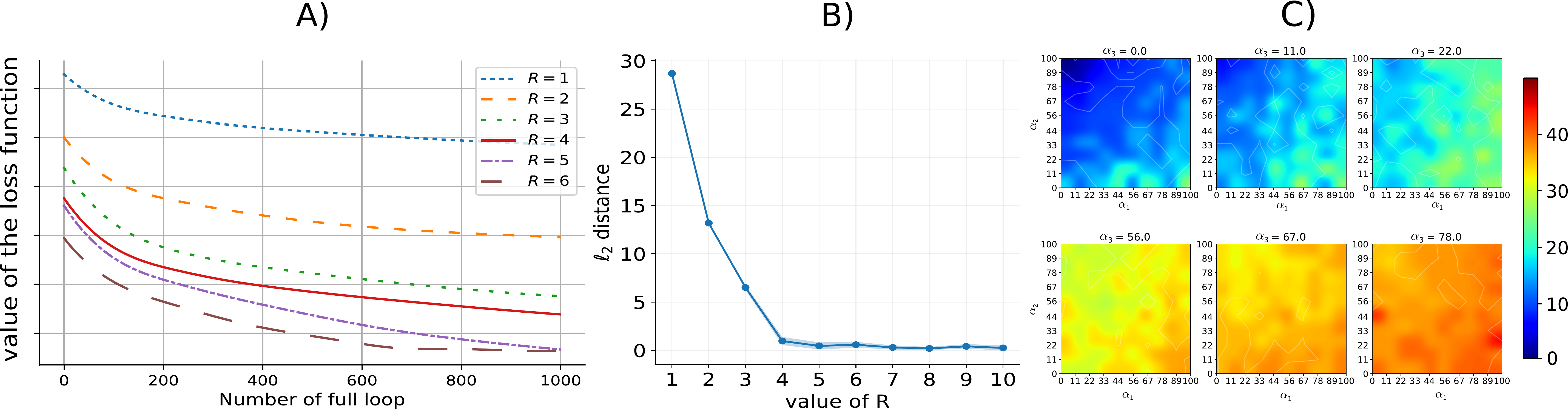}
	\caption{A) Loss \eqref{equ:CSKTRridge} as a function of the rank parameter $R$ and the number of full loops ($\calZ$ and $\calD$ steps). B) $\ell_2$-distance after convergence as a function of the rank parameter $R$ (the true one being $R_*=4$) on $8$ trials C) Heatmap of the $\ell_2$-distance for several hyperparameters values ($R=R_*$).}
	\label{fig:hyperparameter_grid}
\end{figure}

\section{Experiments}\label{sec: exp}
In this section we evaluate our tensor-based dictionary learning framework \textbf{K-CSC} on both synthetic and real-world datasets. We compare our algorithm \textbf{AK-CSC} to state-of-the-art dictionary learning algorithm based on ADMM. Results are for tensor-structured output regression problems on which we report $\ell_2$-distance. All experiments are conducted on a linux laptop with $4$-core $2.5$GHz Intel CPUs using Tensorly \citep{JMLR:v20:18-277}, Sporco \citep{wohlberg2017sporco} and standard python libraries.

\textbf{$\calZ$-step solver.}
As shown in the previous section, the main problem can be rewriten as a regression problem. Hence, to solve the $\calZ$-step, it is possible to use standard tensor regression solvers \citep{zhou2013tensor, li2017near, he2018boosted}. However, it necessitates the construction of an enormous circulant tensor which is not tractable in practice due to memory limitation. In section \ref{sec: model estimation}, we show that the $\calZ$-step necessitates to solve $p$ multi-channel CSC problem with a large amount of channels ($C=\prod_{i=2}^{p} n_i$). While this problem has received only limited attention for many more than three channels, in our experiment, we used the algorithm proposed in \citep{garcia2018convolutional} which turns out to be scalable regarding the value of $C$.

\textbf{Initialization.}
Each activation subproblem is regularized with a $\ell_1$-norm, which induces sparsity on a particular mode. As an example, consider the update of the $Z$ block $([\bZ_{k, i}, \cdots, \bZ_{k, i}]),$ during the $\calZ$-step (see equation \eqref{eq:zstep_1}). From \cite{tibshirani2015statistical}, we know that there exists a value $\alpha^{\max}_{i}$ above which the subproblem solution is always zeros. As $\alpha^{\max}_{i}$ depends on the "atoms" $\tilde{\calD}_{s, :, c}$ and on the multichannel signal $\tilde{\calY}_{:, c}$, its value changes after a complete loop. In particular, its value might change a lot between the initialization and the first loop. This is problematic since we cannot use a regularization $\alpha_i$ above this initial $\alpha^{\max}_{i}$. The standard strategy to initialize univariate CSC methods is to generate Gaussian white noise atoms. However, as these atoms generally poorly correlate with the signals, the initial value of $\alpha^{\max}_{i}$ are low compared to the following ones. To fix this problem, we use the strategy propose in \citep{la2018multivariate} and initialize the dictionary with random parts of the signal. Furthermore, in all the experiments, a zero-padding is added to the initial tensor $\calY$.
\paragraph{Synthetic data.}
To highlight the behavior of \textbf{AK-CSC} and the influence of the hyperparameters, we consider two different scenarios; 1) The rank is unknown 2) The rank is known. In both cases, the true dictionary is given and only the CSC task is evaluated. We generate $K=10$ random tensor atoms of size $\IR^{2 \times 4 \times 8}$ where entries follow a Gaussian distribution with mean $0$ and standard deviation in $[1, 10]$. Each atom is associated to a sparse activation tensor of CP-rank $R^*=4$. The third order tensor $\calY$ generated by model \eqref{equ:convolve_p_rank} is in $\IR^{16 \times 32 \times 64}$. First, we illustrate the convergence of \textbf{AK-CSC} by plotting the loss \eqref{equ:CSKTRridge} as a function of the rank parameter $R$ (Figure \ref{fig:hyperparameter_grid} A)). Convergence is reached after a reasonable number of full loops. It appears that choosing a rank $R$ greater than the true one permits a faster convergence. However, as depicted in Figure \ref{fig:hyperparameter_grid} B), an over estimation of the rank does not increase nor decrease the reconstruction error. Also note that, as expected, the error after convergence drastically decreases when we reach the true rank $R_*$.
The research of the best set of hyperparameters for $\balpha = (\alpha_1, \alpha_2, \alpha_3)$ is done in $[0, 100]^3$. Figure \ref{fig:hyperparameter_grid} C) illustrates how the three hyperparameters related to the sparsity influence the $\ell_2$-distance, i.e. the reconstruction of $\calY$. We see that our method is robust to small modification of the hyperparameters. This is an important property which facilitates the search of the best set of hyperparameters. \\

\paragraph{Color Animated Pictures.} 
We consider a RGB-animated picture composed of $20$ images of size $30 \times 30 \times 3$ that make up the animation of Mario running.
Hence, the animated picture is a $4$-th order tensor of size $\calY \in \IR^{30 \times 30 \times 3 \times 20}$. The objective is to learn a maximum of $K=20$ RGB-animated atoms in $\IR^{20 \times 10 \times 3 \times 3}$ in order to reconstruct $\calY$. We compare the reconstruction of our algorithm \textbf{AK-CSC} when $R=1$ to the ADMM approach without low-rank constraint, which is a classical and efficient CSC solver (see comparative section \citep{wohlberg2016efficient}). Figure \ref{fig:Mario_reconstruct} A) shows that for the same level of sparsity $\balpha$ our method always uses fewer non-zero coefficients and yet provides a better reconstruction. Indeed, the rank constraint on the activations allows to choose more accurate non-zero coefficients. For instance, Figure \ref{fig:Mario_reconstruct}B) shows that even with $2.5$ times less non-zero coefficients, \textbf{AK-CSC} provides a visual better reconstruction.

\begin{figure}[t]
\centering
\includegraphics[width=1.\linewidth]{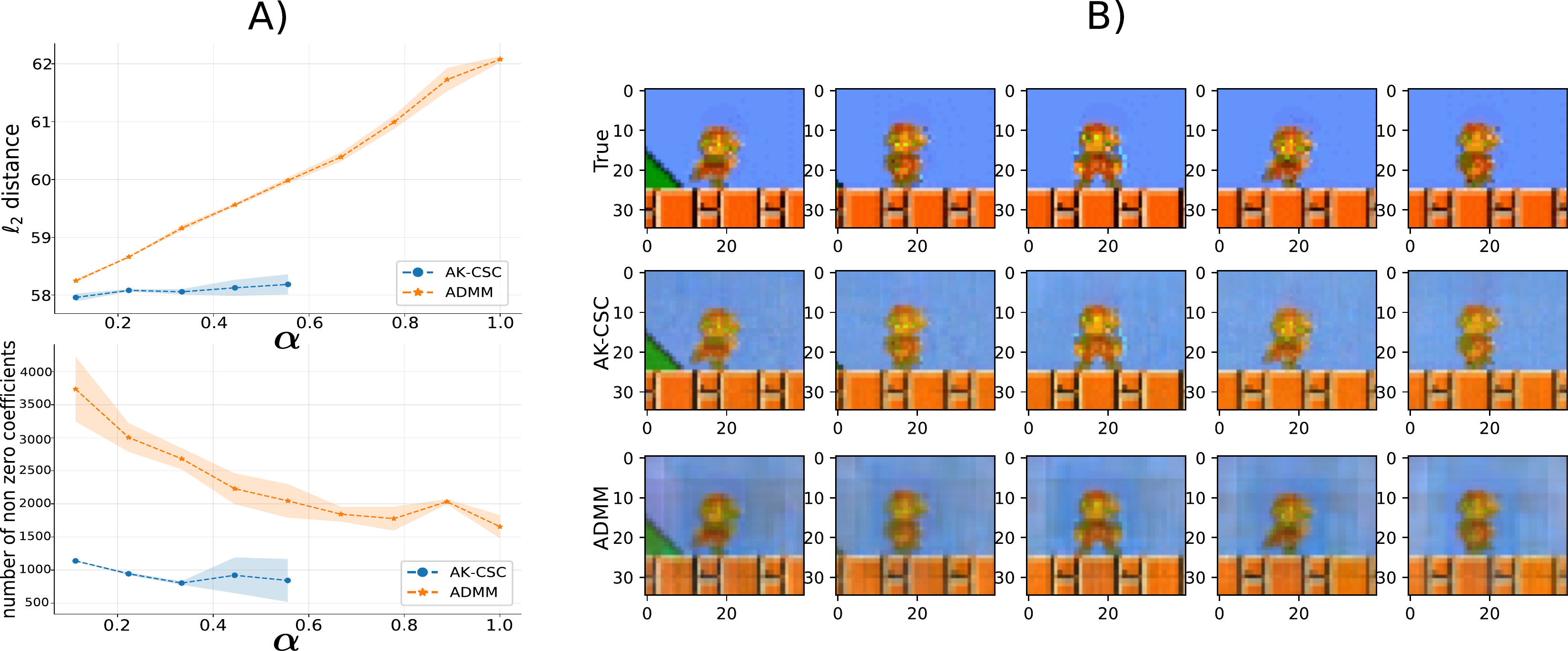}
\caption{A) $\ell_2$-distance and number of non-zero coefficients as a function of the sparsity level $\balpha$ on $20$ trials. B) First $5$ frames of: Original data (top), reconstruction with \textbf{AK-CSC} (middle) and, reconstruction with ADMM (bottom).}
\label{fig:Mario_reconstruct}
\end{figure}
\paragraph{Functional Magnetic Resonance Imaging dataset.} 
In a last example, we analyze functional Magnetic Resonance Imaging data (fMRI) from the Autism Brain Imaging Data Exchange (ABIDE) \citep{di2008functional}. The fMRI data is a third-order tensor in $\IR^{31 \times 37 \times 31}$. We learn $K=20$ atoms of size $\IR^{10 \times 10 \times 10}$ and of rank $R=1$. As in the previous example, we compare the reconstruction of our algorithm \textbf{AK-CSC} to the ADMM approach. The resulting reconstruction is displayed on Figure \ref{fig:fMRI_reconstruct} and the learned atoms are in the supplementary materials. It appears that for the same level of sparsity, the reconstruction performance is more efficient with \textbf{AK-CSC}. Furthermore, the learned atoms are more informative than those with the classical method.

\begin{figure}
	\centering
	\includegraphics[width=1.\linewidth]{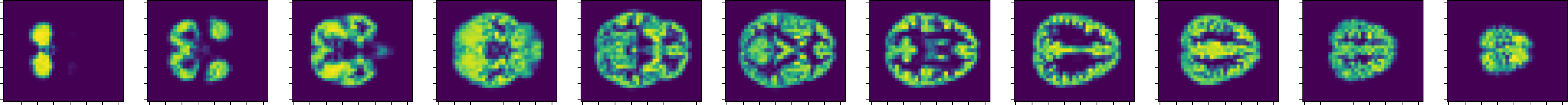}
	\includegraphics[width=1.\linewidth]{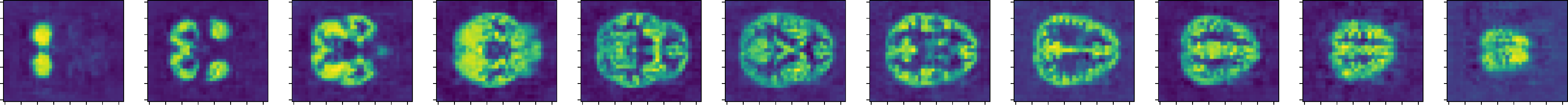}
	\includegraphics[width=1.\linewidth]{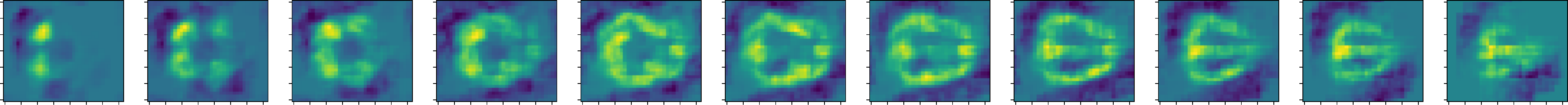}
	\caption{Original data (top), reconstruction with \textbf{AK-CSC} (middle) and, reconstruction with ADMM (bottom).}
	\label{fig:fMRI_reconstruct}
\end{figure}
	\section{Conclusion}
In this paper, we introduced \textbf{K-CSC}, a new multivariate convolutional sparse coding model using tensor algebra that enforces both element-wise sparsity and  low-rankness of the activations tensors. We provided an efficient algorithm based on alternate minimization called \textbf{AK-CSC} to solve this model, and  
we showed that K-CSC can be rewritten as a Kruskal Tensor regression problem.
Finally, we showed that \textbf{AK-CSC} achieves good performances on both synthetic and real data.
		
	\bibliographystyle{unsrtnat}
	\bibliography{sample}

\section{Detailed proof}
We detail the proof of the following proposition from the main paper.
\begin{proposition} The first term of minimization problem ($4$)  can be rewritten as 
	\begin{equation*}
	\Big\lVert \calY - \sum_{k=1}^{K} \calD_k \star_{1, \cdots, p} [\![\bZ_{k,1}, \cdots, \bZ_{k, p}]\!] \Big\rVert_F^2 = \sum^C_{c} \Big\lVert \tilde{\calY}_{:, c} - \sum^S_{s} \tilde{\calD}_{s, :, c} \star \bz_{s}^{(\ell)} \Big\rVert^2_2,
	\end{equation*}
	with $C=\prod_{i=2}^{p} n_i$, $S=K  R$ and $\bz_{s}^{(\ell)} = \big[z_{1, 1}^{(\ell)}, \cdots, z_{1, R}^{(\ell)}, z_{2,1}^{(\ell)}, \cdots, z_{K, R}^{(\ell)} \big]$.
\end{proposition}
\begin{proof}
	In the following, for all $k$, we denote by $\bar{\calZ_k} = \sum_{r=1}^{R}  \bar{z}_{k, r}^{(1)} \circ \cdots \circ \bar{z}_{k, r}^{(p)} \in \mathbb{Y}$ the tensor where we add $0$ on each dimension to reach the one of $\calZ$.
	\begin{align*}
	&\Big\lVert \calY - \sum_{k=1}^{K} \calD_k \star_{1, \cdots, p} \calZ_k \Big\rVert_F^2 \\
	=& \sum_{i_1=1, \cdots, i_p=1}^{n_1, \cdots, n_p} \Big(\calY_{i_1, \cdots, i_p} - \sum_{k=1}^{K}\sum_{j_1=1, \cdots, j_p=1}^{w_1, \cdots, w_p}\calD_{k, j_1, \cdots, j_p} \sum_{r=1}^{R}  \bar{z}_{k, r}^{(1)}(i_1 - j_1) \cdots  \bar{z}_{k, r}^{(p)}(i_p - j_p)\Big)^2 \\
	=& \sum_{i_1=1, \cdots, i_p=1}^{n_1, \cdots, n_p} \Big(\calY_{i_1, \cdots, i_p} - \sum_{k=1}^{K}\sum_{r=1}^{R}\sum_{j_1=1}^{w_1}  \bar{z}_{k, r}^{(1)}(i_1 - j_1) \\ 
	=&\sum_{j_2=1, \cdots, j_p=1}^{w_2, \cdots, w_p}\calD_{k, j_1, \cdots, j_p}  \bar{z}_{k, r}^{(2)}(i_2 - j_2) \cdots  \bar{z}_{k, r}^{(p)}(i_p - j_p)\Big)^2 \\
	=& \sum_{i_1=1, \cdots, i_p=1}^{n_1, \cdots, n_p} \Big(\calY_{i_1, \cdots, i_p} - \sum_{k=1}^{K}\sum_{r=1}^{R}\sum_{j_1=1}^{w_1}  \bar{z}_{k, r}^{(1)}(i_1 - j_1) \Big(\calD_{k; j_1, :, \cdots, :} \star_{2, \cdots, p} z_{k, r}^{(2)} \circ \cdots  \circ z_{k, r}^{(p)}\Big)_{i_2, \cdots, i_p} \Big)^2 \\
	=& \sum_{i_2=1, \cdots, i_p=1}^{n_2, \cdots, n_p} \sum_{i_1=1}^{n_1} \Big(\calY_{i_1, i_2, \cdots, i_p} - \sum_{k=1}^{K}\sum_{r=1}^{R}\sum_{j_1=1}^{w_1}  \bar{z}_{k, r}^{(1)}(i_1 - j_1) \tilde{\calD}_{k; r, j_1, i_2 \cdots, i_p} \Big)^2 \\
	=& \sum_{i_2=1, \cdots, i_p=1}^{n_2, \cdots, n_p} \Big\lVert \calY_{:, i_2, \cdots, i_p} - \sum_{k=1}^{K}\sum_{r=1}^{R} \tilde{\calD}_{k, r, :, i_2 \cdots, i_p} \star  z_{k, r}^{(1)} \Big\rVert^2_2 = \sum^C_{c=1} \Big\lVert \calY_{:, c} - \sum^S_{s=1} \tilde{\calD}_{s, :, c} \star z_{s}^{(1)} \Big\rVert^2_2, \\
	\end{align*}
	where $\tilde{\calD}_{k; r, j_1, i_2 \cdots, i_p} = \Big(\calD_{k; j_1, :, \cdots, :} \star_{2, \cdots, p} z_{k, r}^{(2)} \circ \cdots  \circ z_{k, r}^{(p)}\Big)_{i_2, \cdots, i_p}.$
\end{proof}

\section{Additional results}
\subsection{Synthetic data}
We provide the full heatmap obtained in the \textbf{Synthetic data} section on Figure \ref{fig:hyperparameter_grid4}. This Figure shows how the three hyperparameters related to the sparsity influence the $\ell_2$-distance, i.e. the reconstruction of $\calY$. We see that our method is robust to small modification of the hyperparameters. This is an important property which facilitates the search of the best set of hyperparameters.
\begin{figure}
	\centering
	\includegraphics[width=1.\linewidth]{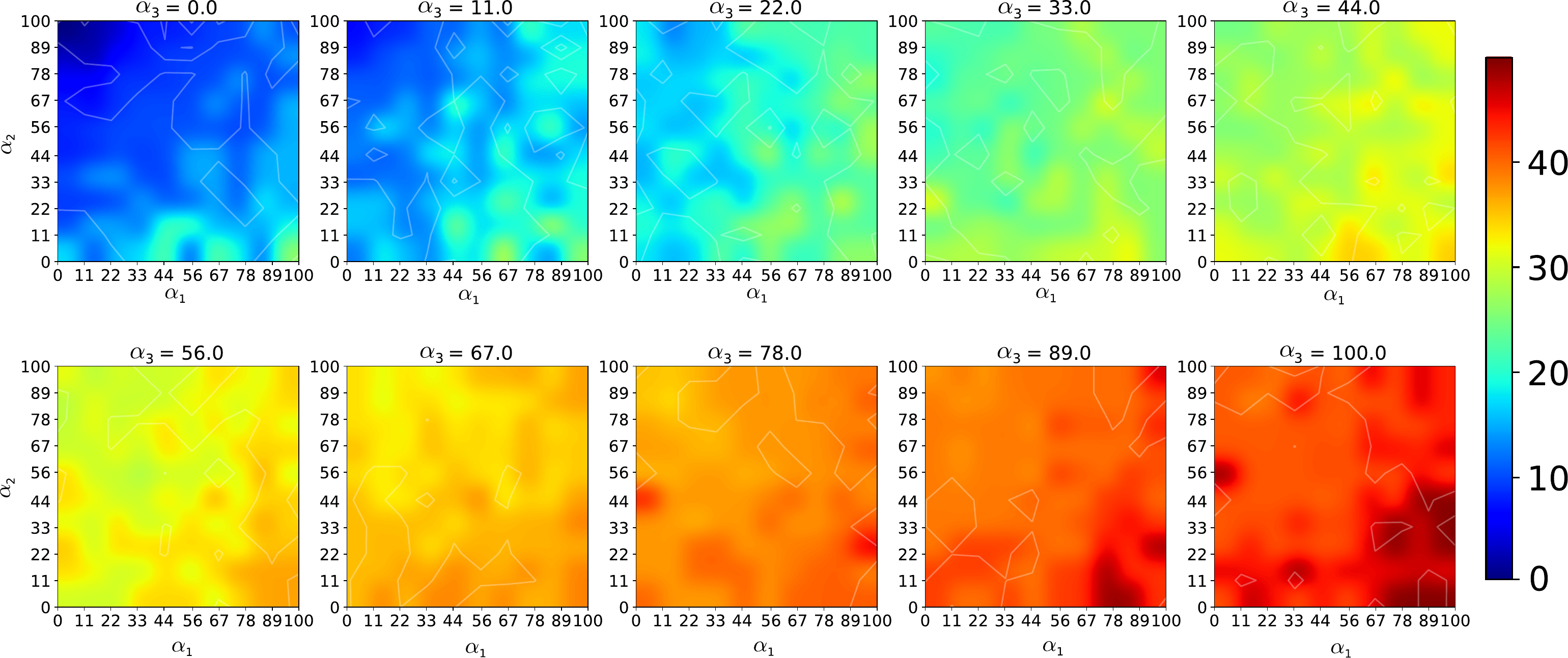}
	\caption{Heatmap of the $\ell_2$-distance between $\calY$ and its reconstruction to tune hyperparameters when considering third order tensor and $R_*=4$.}
	\label{fig:hyperparameter_grid4}
\end{figure}

\subsection{Functional Magnetic Resonance Imaging dataset}
The atoms learned from the main article for functional Magnetic Resonance Imaging (fMRI) are exposed Figure \ref{fig:fMRI_atoms}. We can see that our method learns interesting atoms (on the left) while standard methods failed to use the full dictionary (on the right). Indeed, more than half of the atoms learned by ADMM remain noise.
\begin{figure}[t]
	\centering
	\includegraphics[width=1.\linewidth]{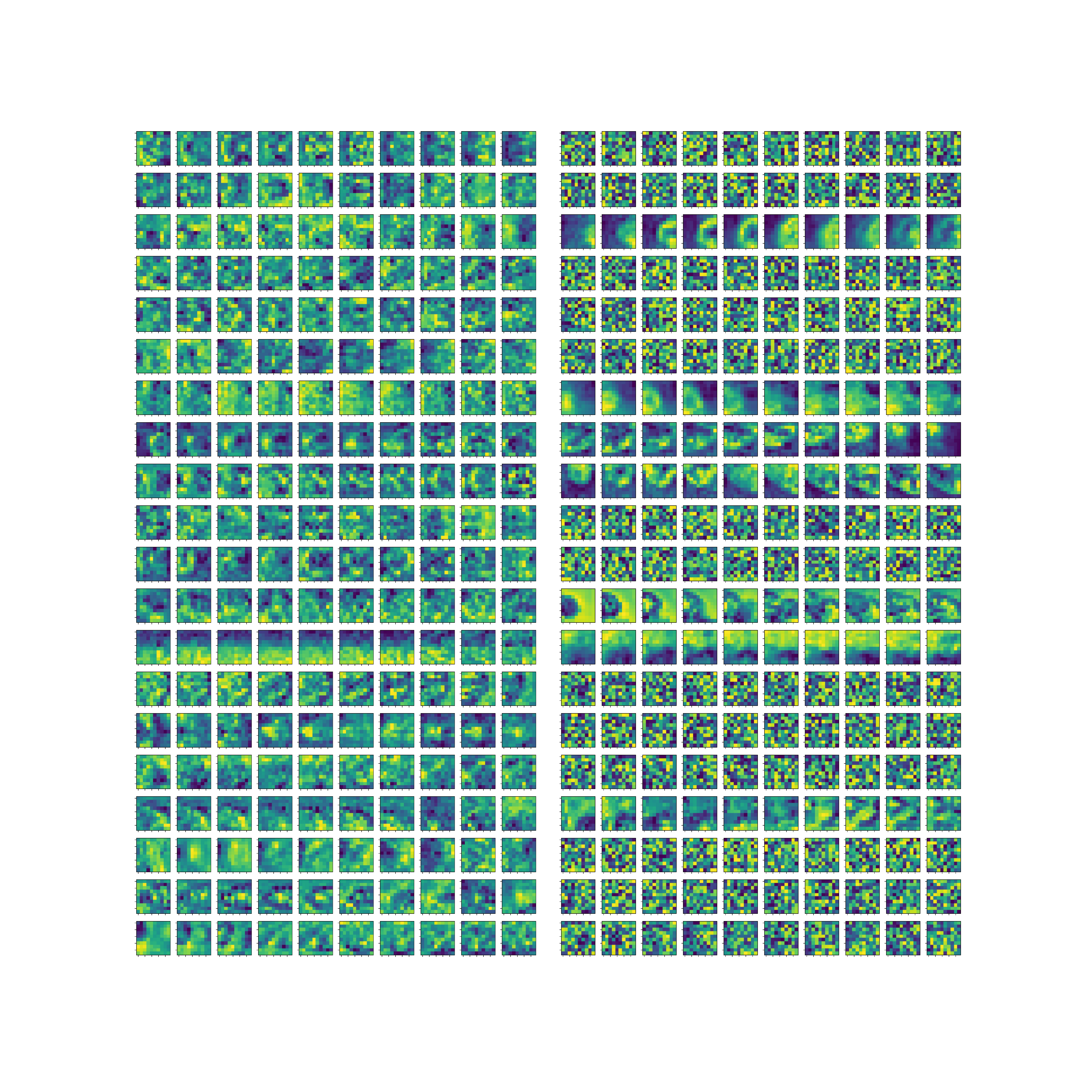}
	\caption{Learned atoms for the fMRI. On the left, the atoms from our method. On the right, atoms from the classical method.}
	\label{fig:fMRI_atoms2}
\end{figure}

\subsection{Color Animated Picture}
To highlight the behavior of \textbf{AK-CSC} with regard to its non-convexity, we performed $20$ trials of the same minimization program. Results of such an experiments is provided on Figure \ref{fig:converge_mario}. This illustration shows that the convergence is globally identical on each trials. Hence, our algorithm seems to be robust.
\begin{figure}[t]
	\centering
	\includegraphics[width=1.\linewidth]{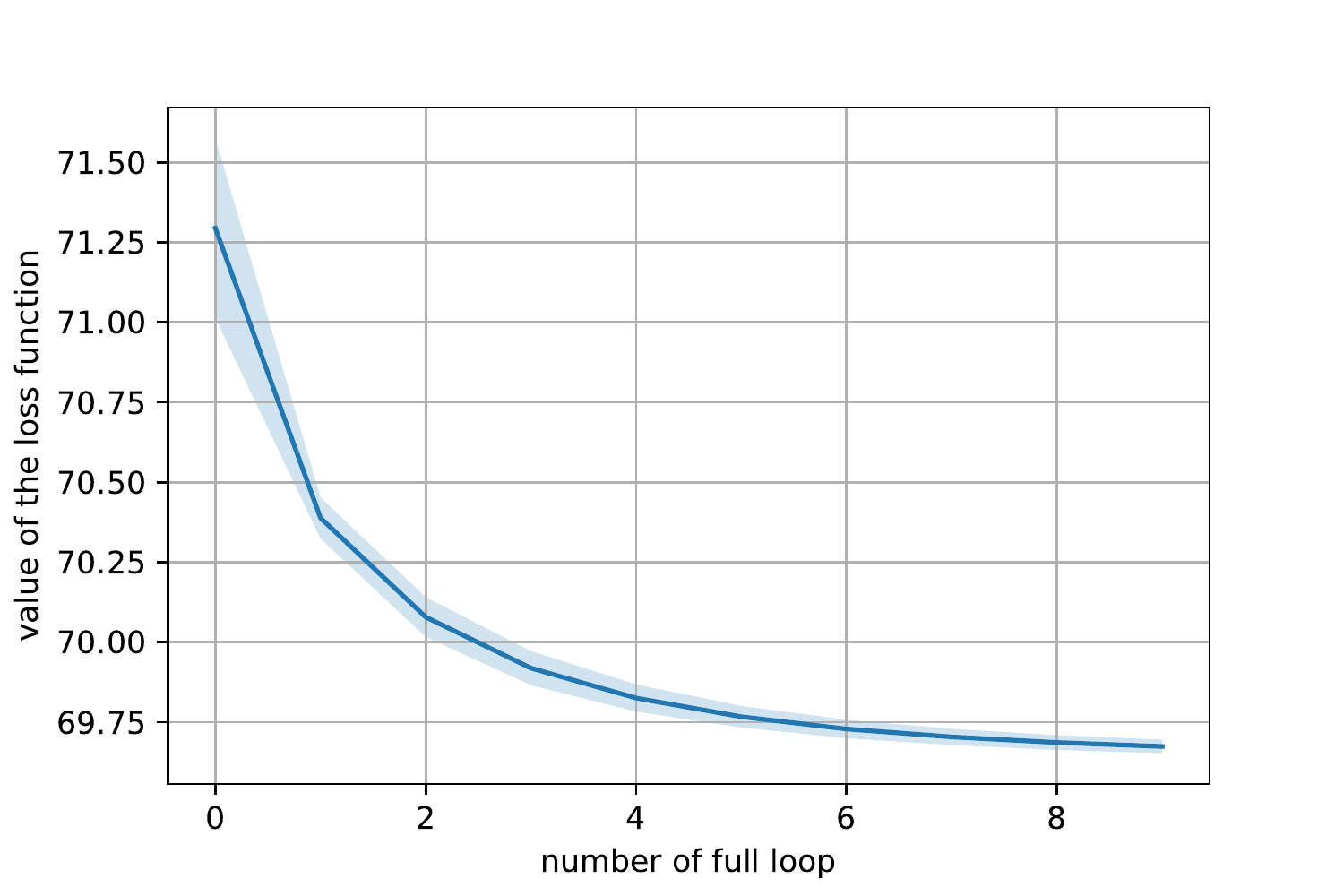}
	\caption{Mean and standard deviation of the evolution of the loss function with AK-CSC regarding the number of full loops ($\calZ$ and $\calD$ steps) on 20 trials.}
	\label{fig:converge_mario2}
\end{figure}

\subsection{Black and White Animated Picture}
In this section, we focus on a black and white version of the animated pictures of Mario. Interestingly, the number of used atoms and of activations is much more smaller for our method. Indeed, it appears that only important atoms are kept as for the other one, the resulting dictionary is redundant. Quantitatively, for a similar error of reconstruction, the number of activations is much smaller. Figure 	\ref{fig:Mario_reconstruct2} shows one example of reconstruction with associated important atoms.

\begin{figure}
	\centering
	\includegraphics[width=1.\linewidth]{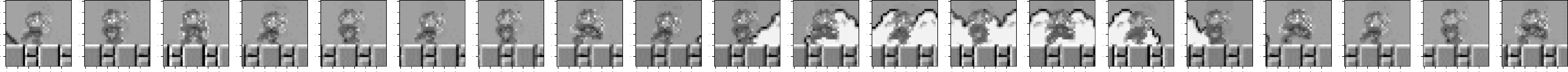}
	\includegraphics[width=1.\linewidth]{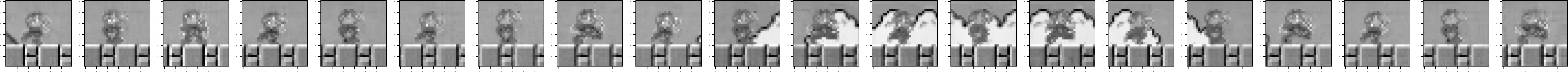}
	\includegraphics[width=1.\linewidth]{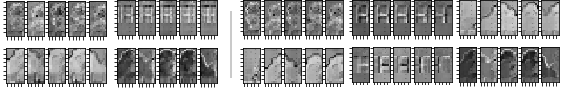}
	\caption{Illustration of the true (on the top) and reconstruct (on the bottom) animated picture. The third picture corresponds to the significant learned animated atoms. On the left, our method and on the right, the classical one.}
	\label{fig:Mario_reconstruct2}
\end{figure}
\end{document}